\documentclass{article}

\usepackage{graphicx} 
\usepackage{subfigure}

\usepackage{natbib}

\usepackage{algorithm}
\usepackage{algorithmic}

\usepackage{hyperref}

\usepackage{amssymb, multirow, paralist}
\usepackage{amsmath}
\usepackage{url}
\usepackage{amssymb}
\usepackage{amsthm}
\newtheorem{thm}{Theorem}
\newtheorem{prop}{Proposition}
\newtheorem{lemma}{Lemma}

\def \R {\mathbb{R}}
\def \Kh {\widehat{K}}
\def \x {\mathbf{x}}
\def \D {\mathcal{D}}
\def \X {\mathcal{X}}
\def \v {\mathbf{v}}
\def \Dh {\widehat{\mathcal D}}
\def \xh {\widehat{\mathbf x}}
\def \u {\mathbf{u}}
\def \lh {\widehat{\lambda}}
\def \Wh {\widehat{W}}
\def \H {\mathcal{H}}

\def \Hk {\H_{\kappa}}
\def \vh {\widehat{\varphi}}
\def \Hh {\widehat{H}}
\def \f {\mathbf{f}}
\def \Phl {\overline{\Phi}}
\def \z {\mathbf{z}}

\title{An Improved Bound for the Nystr\"{o}m Method for Large Eigengap}

\date{}

\author{Mehrdad Mahdavi, Tianbao Yang, and Rong Jin \\ Department of Computer Science \\ Michigan State University \\ \small{\{mahdavim, yangtia1, rongjin\}@msu.edu}}
\begin{document}

\maketitle

\begin{abstract}
We develop an improved bound for the approximation error of the Nystr\"{o}m method under the assumption that there is a large eigengap in the spectrum of kernel matrix. This is based on the empirical observation that the eigengap has a significant impact on the approximation error of the Nystr\"{o}m method. Our approach is based on the concentration inequality of integral operator and the theory of matrix perturbation. Our analysis shows that when there is a large eigengap, we can improve the approximation error of the Nystr\"{o}m method from $O(N/m^{1/4})$ to $O(N/m^{1/2})$ when measured in Frobenius norm, where $N$ is the size of the kernel matrix, and $m$ is the number of sampled columns.
\end{abstract}

\section{Introduction}

The Nystr\"{o}m method has been used in kernel learning to approximate large kernel matrices~\citep{Fowlkes04spectralgrouping,Platt04fastembedding,kuma-2009-sampling,kai-2008-improved,Williams01usingthe,cortes-2010-nystrom,talwalkar-2008-large,Drineas05onthe,silva-2003-gloal,belabbas-2009-spectral,talwalkar-2010-matrix}. In order to evaluate the quality of Nystr\"{o}m method, we typically bound the norm of the difference between the original kernel matrix and the low rank approximation created by the Nystr\"{o}m method. Both the Frobenius norm and the spectral norm have been used to bound the difference between matrices~\citep{Drineas05onthe}. The key result from~\citep{Drineas05onthe} is that besides the intrinsic error due to the low rank approximation, the additional error caused by the Nystr\"{o}m method is $O(N/m^{1/4})$ when measured in Frobenius norm, provided that the diagonal elements of kernel matrix is bounded by a constant. In this work, we consider the case when there is a large eigengap in the spectrum of the kernel matrix, a scenario that has been examined in many studies of kernel learning~\citep{Bach:CSD-03-1249,Luxburg:2007:tutorial,Azran:Ghahramani:06,DS_AOS_09}. Given sufficiently large eigengap, we are able to improve the bound for the additional approximation error caused by the Nystr\"{o}m method to $O(N/m^{1/2})$ when measured in Frobenius norm. The key techniques used in our analysis are the concentration inequality of integral operator~\citep{smale-2009-geometry} and matrix perturbation theory~\citep{Stewart90}.

Our paper is structured as follows:  in section  \ref{sec-back-moti}, we demonstrate a discrepancy between the theoretical and experimental approximation error of the Nystr\"{o}m method  that motivates our work to improve the existing bounds. Section \ref{sec-bounds} introduces the problem formally and proves the bounds. Finally, section \ref{sec-conclusion} concludes the paper.

\section{Background and Motivation}
\label{sec-back-moti}

\begin{figure*}[t]
\centering\hspace*{-0.2in}
\subfigure[MNIST]{\includegraphics[width=0.4\textwidth]{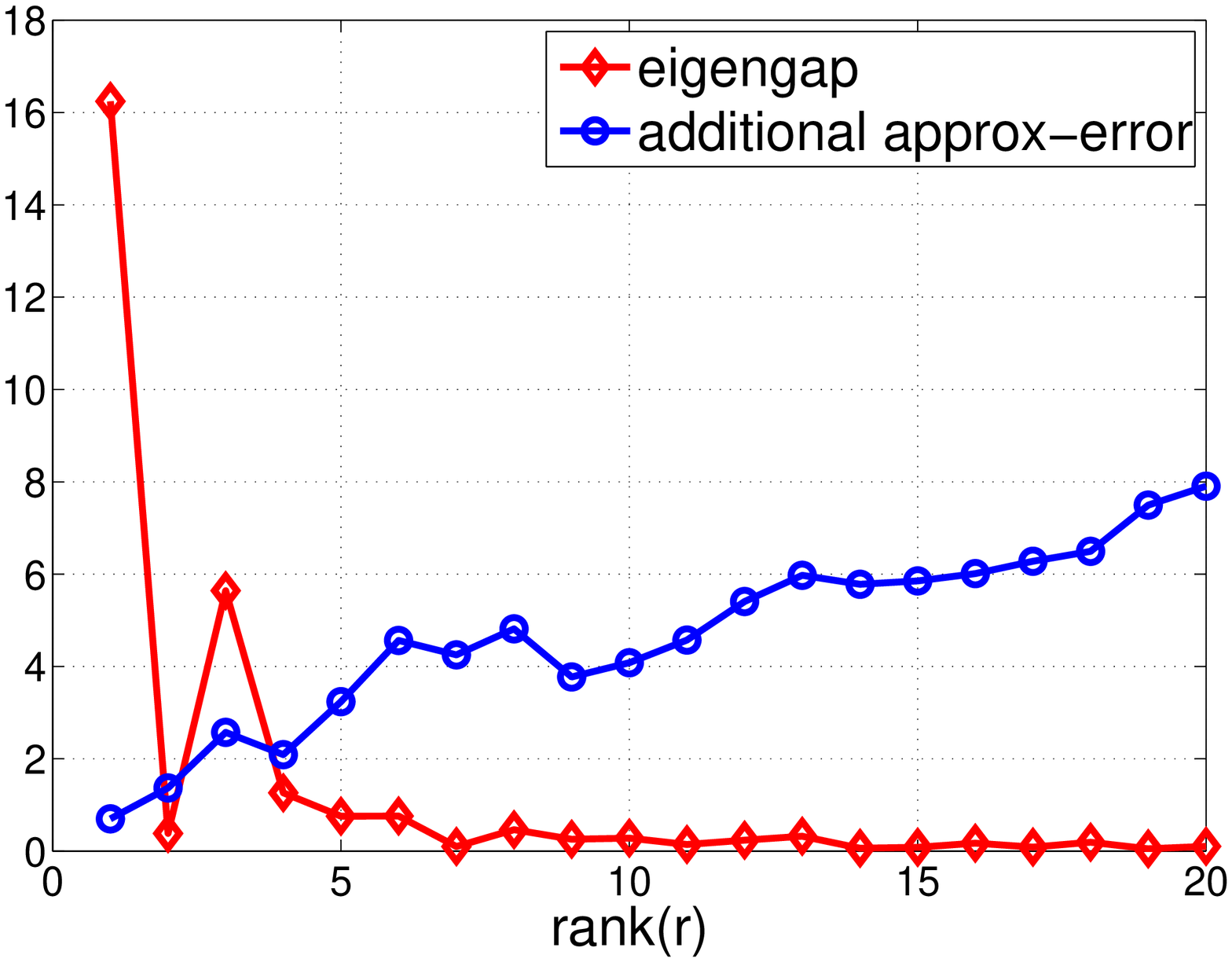}}\hspace*{-0.1in}
\subfigure[a7a]{\includegraphics[width=0.4\textwidth]{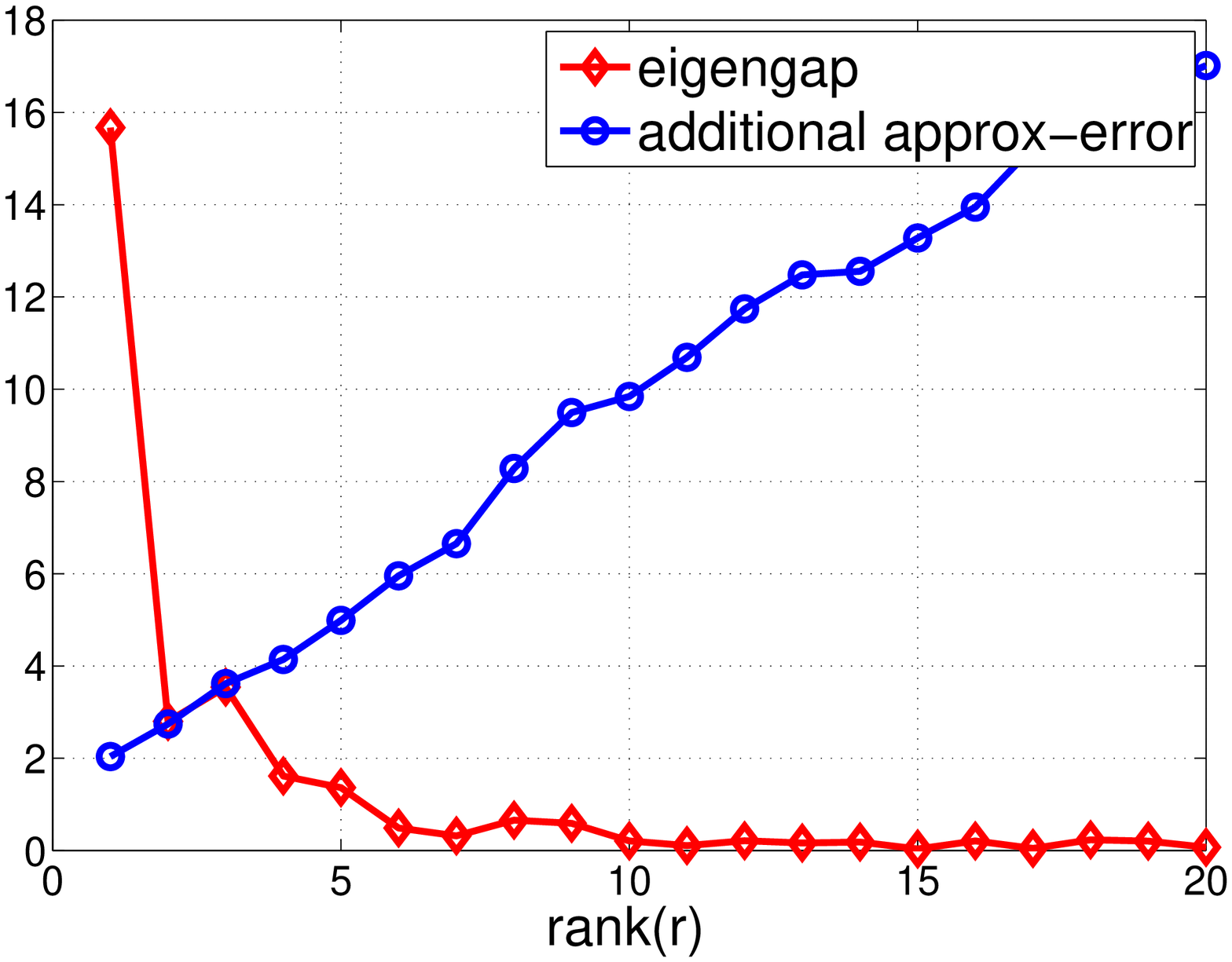}}\hspace*{-0.1in}\\
\subfigure[diabetes]{\includegraphics[width=0.4\textwidth]{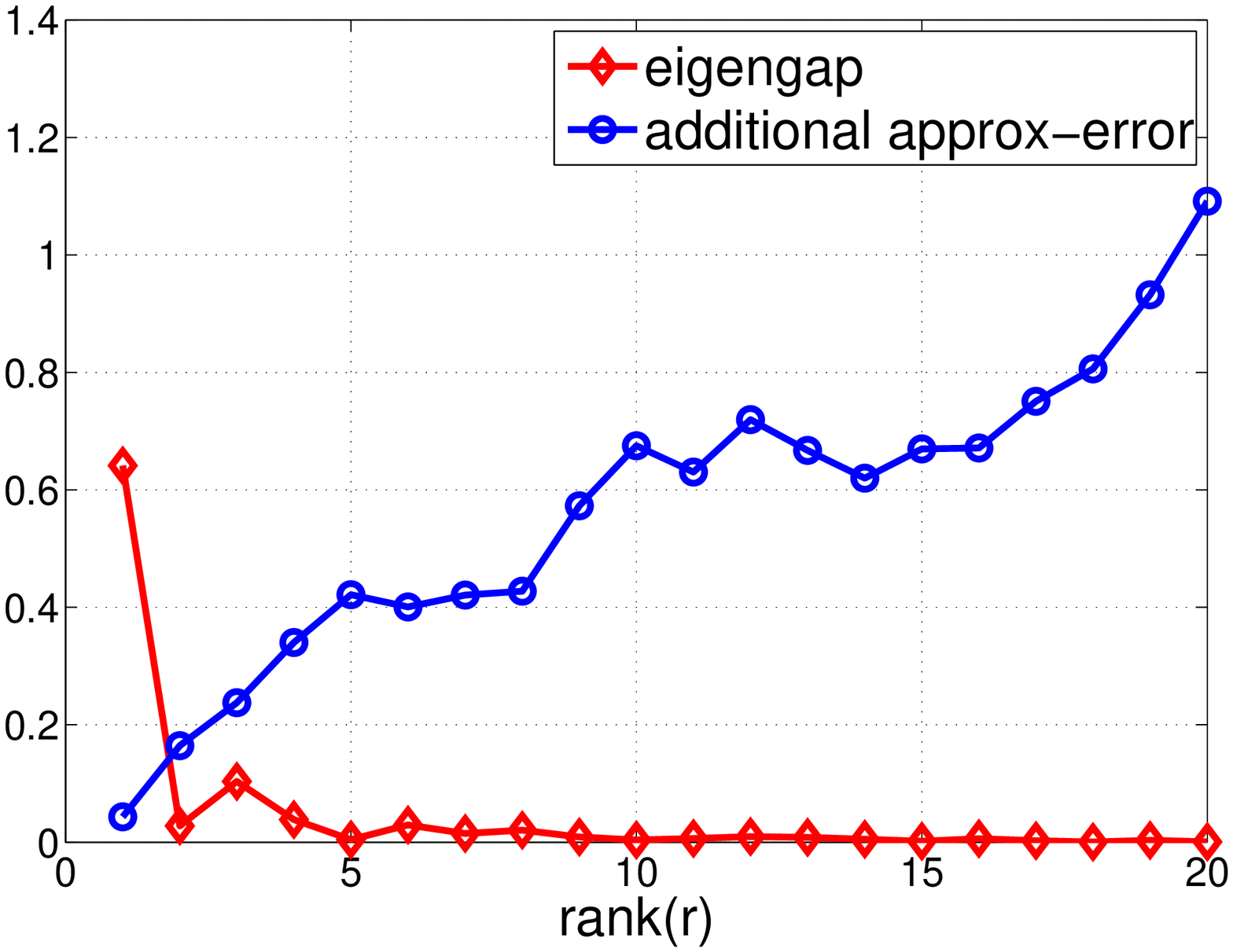}}\hspace*{-0.1in}
\subfigure[CPU]{\includegraphics[width=0.4\textwidth]{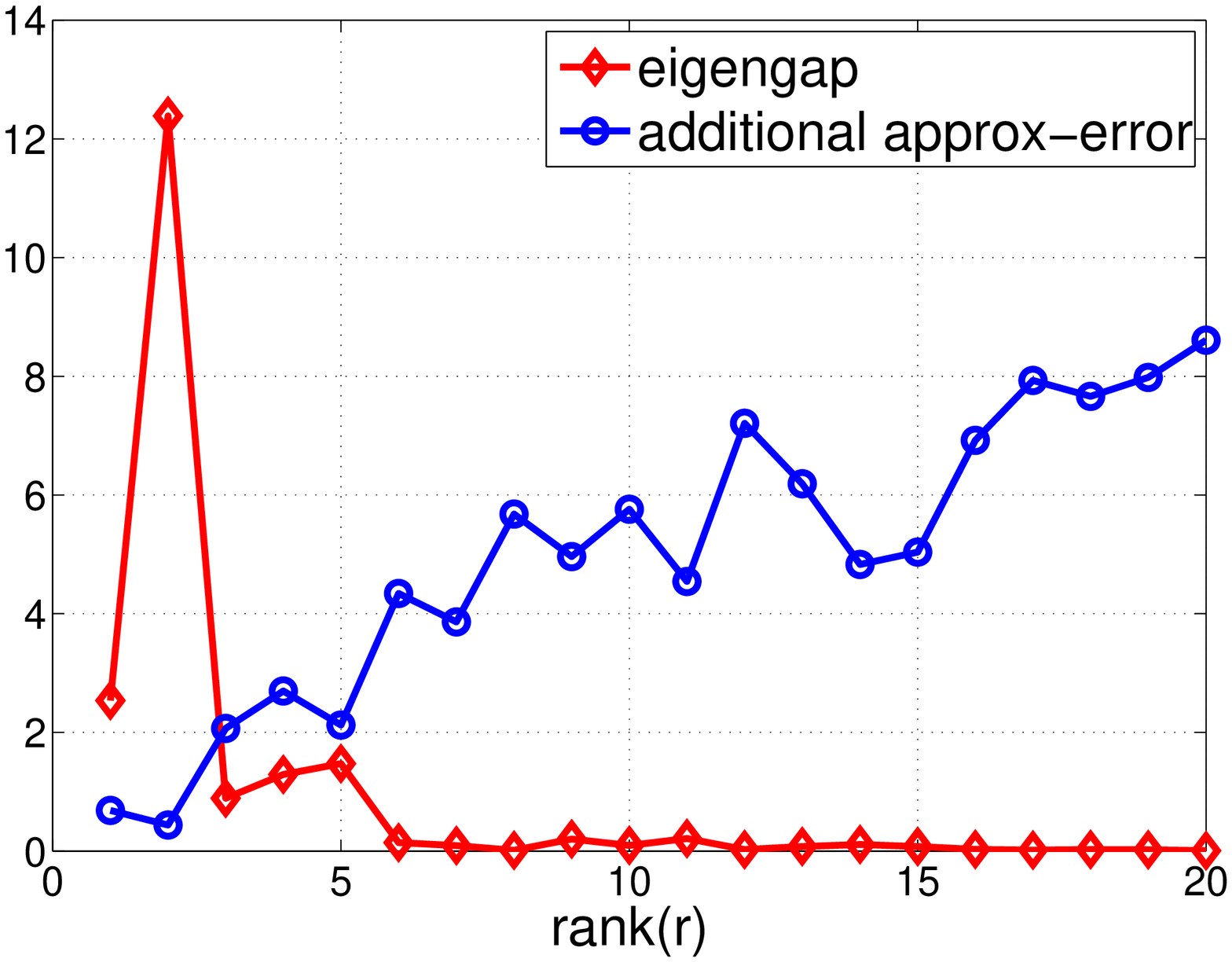}}
\vspace*{-0.1in}\caption{{\it Additional} approximation error $\|K-\Kh_r\|_F-\|K - K_r\|_F$ and eigengap $\lambda_r-\lambda_{r+1}$. Both the additional approximation error and eigengap are scaled appropriately so that they fall into the same range.}\label{fig:app}
\end{figure*}

The Nystr\"{o}m method was first suggested in~\citep{Williams01usingthe} to improve the computational efficiency of Gaussian process. It was then adopted by a number of studies to improve the computational efficiency of kernel learning~~\citep{Fowlkes04spectralgrouping,Platt04fastembedding,kuma-2009-sampling,kai-2008-improved,talwalkar-2008-large,Drineas05onthe,silva-2003-gloal,cortes-2010-nystrom,belabbas-2009-spectral,talwalkar-2010-matrix}. Several analysis have been presented to bound the approximation error by the Nystr\"{o}m method~\citep{Drineas05onthe,kuma-2009-sampling,belabbas-2009-spectral,talwalkar-2010-matrix}. Most of them are based on the result from~\citep{Drineas05onthe} except for~\citep{talwalkar-2010-matrix} whose analysis is limited to low rank kernel matrices and does not apply to the general case.

Let $K \in \R^{N\times N}$ be the kernel matrix to be approximated. Let $K_r$ be the $r$-rank best approximation of kernel matrix $K$, and let $\Kh_r$ be an approximate kernel matrix of rank $r$ generated by the Nystr\"{o}m method. Assume $K_{i,i}\leq 1$ for any $i \in [N]$. Let $m$ be the number of columns uniformly sampled from $K$ used to construct $\Kh_r$. Both Frobenius norm and spectral norm are used to bound the difference between $K$ and $\Kh_r$. We note that it is important to derive the approximation errors measured in both norms as they have different implications. According to~\citep{cortes-2010-nystrom}, the approximation error measured in spectral norm is closely related to the generalized performance of kernel classifiers. On the other hand, the approximation error measured in Frobenius norm have found applications in kernel PCA \cite{Scholkopf:1998:NCA:295919.295960}, low dimensional manifold embedding~\cite{Belkin:2001}, spectral clustering~\cite{Fowlkes:2004:SGU:960255.960312,ChittaJHJ11}. Improving the bound in the Frobenius norm will help us better understand the application of the Nystr\"{o}m method to those domains.

\citet{Drineas05onthe} shows that with a high probability,  we have
\begin{eqnarray}
    \|K - \Kh_r\|_2 & \leq & \|K - K_r\|_2 + O\left(\frac{N}{\sqrt{m}}\right), \label{eqn:bound-s-1} \\
    \|K - \Kh_r\|_F  & \leq & \|K - K_r\|_F + O\left(\frac{N}{m^{1/4}}\right), \label{eqn:bound-f-1}
\end{eqnarray}
where $\|\cdot\|_2$ and $\|\cdot\|_F$ stand for the spectral norm and  Frobenius norm  of a matrix, respectively. Compared to the bound in spectral norm in (\ref{eqn:bound-s-1}), the bound measured in Frobenius norm is significantly worse in terms of $m$, with the convergence rate of $O(m^{-1/4})$. The difference between the two bounds in (\ref{eqn:bound-s-1}) and (\ref{eqn:bound-f-1}) leads to the following question:

{\it Under what scenario it is possible to improve the convergence rate of the bound in Frobenius norm to that of the bound measured in the spectral norm.}


To this end, we first examine empirically the {\it additional} approximation error $\|K - \Kh_r\|_F - \|K - K_r\|_F$. Note that we intentionally remove $\|K - K_r\|_F$ from the approximation error because $\|K - K_r\|_F$ provides the lower bound for any approximation with matrix of rank $r$. Four UCI datasets are used in this empirical study, i.e., MNIST\footnote{\url{http://yann.lecun.com/exdb/mnist/}}, a7a, diabetes\footnote{\url{http://www.csie.ntu.edu.tw/~cjlin/libsvmtools}}, CPU\footnote{\url{http://archive.ics.uci.edu/ml/datasets/}}. The RBF kernel $\kappa(\x, \x') = \exp(-\lambda\|\x - \x'\|_2^2/d^2)$ is used, where $d^2$ is the average distance square between any two examples and $\lambda = 10$. The blue curves with legend $\circ$ in Figure~\ref{fig:app} show how the additional approximation error $\|K - \Kh_r\|_F - \|K - K_r\|_F$ varies according to the rank $r$. The overall trend, as indicated in Figure~\ref{fig:app}, is that the higher the rank, the larger the additional approximation error tends to be. In order to explain the dependence of the approximation error on rank, we examine the distribution of eigengap $\lambda_{r} - \lambda_{r+1}$ over the rank. The red curves with legend $\diamond$ in Figure~\ref{fig:app} show how the eigengap $\lambda_r - \lambda_{r+1}$ varies over the rank. Overall, we observe that the larger the rank, the smaller the eigengap. By combining the two observations, we conjecture that there is a strong dependence between the eigengap and the approximation error of the Nystr\"{o}m method. This motivates us to develop an eigengap dependent approximation error bound for the Nystr\"{o}m method. Our analysis show that when the eigengap $\lambda_r - \lambda_{r+1}$ is sufficiently large, the approximation error of the Nystr\"{o}m method, measured in Frobenius norm, can be improved to $O(N/\sqrt{m})$, i.e.
\begin{align*}
    \|K - \Kh_r\|_F  \leq  \|K - K_r\|_F + O\left(\frac{N}{\sqrt{m}}\right).
\end{align*}

We note that although the concept of eigengap has been exploited in many studies of kernel learning~\citep{Bach:CSD-03-1249,Luxburg:2007:tutorial,Azran:Ghahramani:06,DS_AOS_09}, to the best of our knowledge, this is the first time it has been incorporated in the analysis of the Nystr\"{o}m method.

In the development of the Nystr\"{o}m method, another important issue is how to sample the columns in the kernel matrix. We restrict our analysis to the uniform sampling. Although different sampling approaches have been suggested for the Nystr\"{o}m method~\citep{Drineas05onthe,kuma-2009-sampling,kai-2008-improved,belabbas-2009-spectral}, according to~\citep{kuma-2009-sampling}, for real-world datasets, uniform sampling seems to be the most efficient and gives comparable performance to the other sampling approaches. We notice that in~\citep{belabbas-2009-spectral}, the authors show a significantly better approximation bound for the Nystr\"{o}m method, both theoretically and empirically, when sampling the columns based on the determinant of the submatrix formed by the selected columns and rows, which is also referred to as determinantal processes~\citep{hough-2006-determinantal}. It is however important to point out that the determinantal process is usually computationally expensive as it requires computing the determinant of the submatrix for the selected columns/rows, making it unsuitable for the case when a large number of columns are needed to be sampled.

\section{Approximation Error Bound by the Nystr\"{o}m Method}
\label{sec-bounds}
Let $\D = \{\x_1, \ldots, \x_N\}$ be a collection of $N$ samples, and $K = [\kappa(\x_i, \x_j)]_{N\times N}$ be the kernel matrix for the samples in $\D$, where $\kappa(\cdot, \cdot)$ is a kernel function. For simplicity, we assume $\kappa(\x, \x) \leq 1$ for any $\x \in \X$. Let $\Hk$ be the Reproducing Kernel Hilbert Space (RKHS) endowed with kernel $\kappa(\cdot, \cdot)$. We denote by $(\v_i, \lambda_i), i=1, \ldots, N$ the eigenvectors and eigenvalues of $K$ ranked in the descending order of eigenvalues. Define $V=(\v_1,\cdots, \v_N)$ and $V_{i,j}=[V]_{i,j}$. In order to build the low rank approximation of kernel matrix $K$ of rank $r$, the Nystr\"{o}m method first samples $m < N$ examples randomly from $\D$, denoted by $\Dh = \left\{\xh_1, \ldots, \xh_m\right\}$. It then computes a sample kernel matrix $\Kh = [\kappa(\xh_i, \xh_j)]_{m\times m}$. Let $(\u_i, \lh_i), i=1, \ldots, r$ be the first $r$ eigenvalues and eigenvectors of matrix $\Kh$, and let $U=(\u_1,\cdots, \u_r)$, $U_{i,j}=[U]_{i,j}$. We assume $\lh_r > 0$ is strictly positive and define matrix $\Wh$ as
\begin{eqnarray*}
    \Wh = \sum_{i=1}^r \frac{1}{\lh_i} \u_i \u_i^{\top}. \label{eqn:wh}
\end{eqnarray*}
The approximate low rank matrix $\Kh_r$, computed by the Nystr\"{o}m method, is given by
\begin{eqnarray*}
\Kh_r = K_b\Wh K_b^{\top} \label{eqn:nystrom},
\end{eqnarray*}
where $K_b = [\kappa(\x_i, \xh_j)]_{N\times m}$ measures the similarity between the samples in $\D$ and $\Dh$. As already mentioned, we focus on the scenario when the eigengap $\lambda_r - \lambda_{r+1}$ is sufficiently large~\footnote{The precise definition of large eigengap will be given later}. Our analysis is mainly based on the concentration inequality of integral operator~\citep{smale-2009-geometry} and matrix perturbation theory~\citep{Stewart90}.

\subsection{Preliminaries} \label{sec:preliminaries}
We define an integral operator $L_N$ and $L_m$ based on the samples in $\D$ and $\Dh$, respectively, as
\begin{align*}
    &L_N[f](\cdot) = \frac{1}{N}\sum_{i=1}^N \kappa(\x_i, \cdot) f(\x_i), \\
    & L_m[f](\cdot) = \frac{1}{m}\sum_{i=1}^m \kappa(\xh_i, \cdot) f(\xh_i),
\end{align*}
where $f \in \Hk$ is any function in $\Hk$. The eigenvalues of the integral operator $L_N$ and  $L_m$, according to~\citep{smale-2009-geometry}, are $\lambda_i/N, i\in[N]$ and $\lh_i/m, i\in[m]$, respectively. Let $\varphi_1(\cdot), \ldots, \varphi_N(\cdot)$ be the corresponding eigenfunctions of $L_N$ that are normalized by functional norm, i.e., $\langle \varphi_i, \varphi_j \rangle_{\Hk} = \delta(i,j), \quad \forall (i, j)\in[N]\times[N]$. According to~\citep{smale-2009-geometry}, the eigenfunctions are given by
\begin{eqnarray}
    \varphi_j(\cdot) = \frac{1}{\sqrt{\lambda_j}}\sum_{i=1}^N V_{i,j} \kappa(\x_i, \cdot), j\in[N]. \label{eqn:varphi}
\end{eqnarray}
Using the eigenfunctions expressed in (\ref{eqn:varphi}), we can write $\kappa(\x_j, \cdot), j\in[N]$ as
\begin{align}\label{eqn:keig}
    \kappa(\x_j, \cdot) &= \sum_{i=1}^N  \sqrt{\lambda_i} V_{j, i}\varphi_i(\cdot)=\sum_{i=1}^N\langle\kappa(\x_j, \cdot), \varphi_i\rangle \varphi_i.
\end{align}
It is easy to verify that $L_N$ can be written in the base of $\varphi_i, i\in[N]$ by
\begin{align}\label{eqn:eigL}
L_N(f)(\cdot) = \frac{1}{N}\sum_{i=1}^N\lambda_i\varphi_i\langle f, \varphi_i\rangle_{\Hk}.
\end{align}

Let $\vh_j, j\in[m]$ be the corresponding eigenvectors  of the integral operator $L_m$. Similar to $L_N$, the eigenfunction $\vh_i$ is given by
\begin{eqnarray}
    \vh_j(\cdot) = \frac{1}{\sqrt{\lh_j}}\sum_{i=1}^mU_{i,j} \kappa(\xh_i, \cdot). \label{eqn:vh}
\end{eqnarray}

We define the Hilbert Schmidt norm  of operator $L: \mathcal H_k\rightarrow \mathcal H_k$ by
 \begin{align}\label{eqn:hs}
\|L\|_{HS}= \sqrt{\sum_{i, j=1}^N \langle \varphi_i, L\varphi_j\rangle_{\Hk}^2}.
\end{align}
Let $\|L\|_2$ denote the spectral norm of operator $L$ defined by
\begin{align*}
\|L\|_2 =  \max_{\|f\|_{\Hk}\leq 1}\langle f, Lf\rangle_{\Hk}.
\end{align*}
where $\langle\cdot, \cdot \rangle_{\Hk}$ denotes the inner product in Hilbert space $\Hk$. In the sequel, we use $\langle\cdot, \cdot\rangle$ for short.

We state the concentration inequality about  the two integral operators in the following.
\begin{lemma}\label{lemma:1}(Proposition 1~\citep{smale-2009-geometry})
Let $\xi$ be a random variable on $(\X, P_{\X})$ with values in a Hilbert space $(\H,\|\cdot\|)$. Assume $\|\xi\| \leq M < \infty$ almost sure. Then with a probability at least $1 - \delta$, we have
\[
    \left\|\frac{1}{m}\sum_{i=1}^m \xi(\x_i) - \mathrm E[\xi]\right\| \leq \frac{4M\ln(2/\delta)}{\sqrt{m}}.
\]
\end{lemma}

\begin{thm} \label{thm:concentration}
With a probability $1 - \delta$, we have
\[
\|L_N - L_m\|_{HS} \leq \frac{4\ln(2/\delta)}{\sqrt{m}},
\]
where $\|L\|_{HS}$ is defined in equation~(\ref{eqn:hs}).
\end{thm}
\begin{proof}
Define $\xi(\xh_i)$ as a rank one linear operator, i.e.,
\[
    \xi(\xh_i)[f](\cdot) = \kappa(\xh_i, \cdot) f(\xh_i).
\]
Apparently, $L_m = \frac{1}{m}\sum_{i=1}^m \xi(\xh_i)$ and $\mathrm{E}[\xi(\xh_i)] = L_N$. Let $\|\cdot\|_{HS}$ be the norm used in Lemma 1. We complete the proof by using the result from Lemma 1 and the fact
\begin{align*}
\|\xi(\xh_k)\|_{HS}&= \sqrt{\sum_{i,j=1}^N\langle\varphi_i, \kappa(\xh_k,\cdot)\varphi_j(\xh_k) \rangle^2}\\
&=\sqrt{\sum_{i,j=1}^N\varphi_i(\xh_k)^2\varphi_j(\xh_k)^2}= \kappa(\xh_k, \xh_k)\leq 1,
\end{align*}
where the last equality follows equation~(\ref{eqn:keig}).
\end{proof}
%
%
\subsection{Bounding the Approximation Error by Operator Norm}

Based on the first $r$ eigenfunctions of $L_N$ and $L_m$, we define two additional linear operators $H_r$ and $\Hh_r$ as
\begin{align*}
    &H_r[f](\cdot) = \sum_{i=1}^r \varphi_i(\cdot) \langle \varphi_i, f \rangle,\\
    &\Hh_r[f](\cdot) = \sum_{i=1}^r \vh_i(\cdot) \langle \vh_i, f \rangle. \label{eqn:H}
\end{align*}
The following lemma relates $H_r$ and $\Hh_r$ to matrices $K_r$ and $\Kh_r$, respectively.
\begin{prop}
Assume $\lh_r > 0$ and $\lambda_r > 0$. We have  for any $(i,j)\in[N]\times[N]$
\begin{align*}
&\left[\Kh_r \right]_{i,j} = \langle \kappa(\x_i, \cdot), \Hh_r \kappa(\x_j, \cdot) \rangle,\\\
&\left[K_r\right]_{i,j} = \langle \kappa(\x_i, \cdot), H_r\kappa(\x_j, \cdot)\rangle.
\end{align*}
\end{prop}
\begin{proof}
By the definition of $\Hh_r$ and equation~(\ref{eqn:vh}),  we have
\begin{align*}
&\langle \kappa(\x_i, \cdot), \Hh_r \kappa(\x_j, \cdot) \rangle\\
 &= \sum_{k=1}^r \frac{1}{\lh_k}\langle \kappa(\x_i, \cdot), \vh_k\rangle\langle \kappa(\x_j, \cdot), \vh_k \rangle \\
& =  \sum_{a, b=1}^m \sum_{k=1}^r \frac{1}{\lh_k}U_{a,k} U_{b,k} \langle \kappa(\x_i, \cdot), \kappa(\xh_a, \cdot)\rangle \langle \kappa(\x_j, \cdot), \kappa(\xh_b, \cdot)\rangle \\
& =  \sum_{a, b=1}^m \sum_{k=1}^r \frac{U_{a,k} U_{b,k}}{\lh_k} [K_b]_{i, a} [K_b]_{j, b}\\
& = \sum_{a, b = 1}^m [K_b]_{i,a}\Wh_{a, b} [K_b]_{j, b}  =  [K_b \Wh K_b^{\top}]_{i,j} = [\Kh_r]_{i,j}.
\end{align*}
Using the fact that $K_r = KWK$, where
\begin{eqnarray*}
W = \sum_{i=1}^r \frac{1}{\lambda_i} \v_i\v_i^{\top} \label{eqn:w},
\end{eqnarray*}
we apply the same proof to $K_r$.
\end{proof}

Next, we will relate $\|K_r - \Kh_r\|_F$ to $\Delta H= H_r - \Hh_r$.  Note that $L_N$ and $H_r, \Hh_r$ are self-adjoint operators, and so is $\Delta H$.  In the proof of Theorem~\ref{thm:1}, we repeatedly use
\begin{align*}
\langle f, \Delta H g\rangle & = \langle \Delta H f, g\rangle,\\
NL_N(f) &= \sum_{i=1}^N \kappa(\x_i,\cdot) \langle f, \kappa(\x_i, \cdot)\rangle.
\end{align*}
\begin{thm} \label{thm:1}
Assume $\lh_r > 0$ and $\lambda_r > 0$. We have
\begin{eqnarray*}
\|\Kh_r - K_r\|_F \leq \sqrt{\lambda_1 N} \|\Delta H\|_2 \leq N\|\Delta H\|_2.
\end{eqnarray*}
\end{thm}
The proof can be found in the Appendix A. As indicated by Theorem~\ref{thm:1}, to bound $\|\Kh_r - K_r\|_F$, the key is to bound the spectral norm of operator $\Delta H$.

\subsection{Bounding the Operator Norm by Matrix Perturbation Theory}
Our next goal is to bound the spectral norm of $\Delta H$. To this end, we assume a large eigengap between $\lambda_r$ and $\lambda_{r+1}$, i.e., $\Delta = (\lambda_r - \lambda_{r+1})/N$ is sufficiently large. Note that we normalize $\lambda_r - \lambda_{r+1}$ by $N$, the size of dataset $\D$, when defining $\Delta$. Eigengap has the key quantity for the application of matrix perturbation theory~\citep{Stewart90}. The following perturbation result from~\citep{Stewart90} forms the foundation of our analysis~\footnote{We simplify the statement to make it better fit with our objective}.
\begin{thm} \label{thm:2} (Theorem 2.7 of Chapter 6~\citep{Stewart90})
Let $(\lambda_i, \v_i), i \in [n]$ be the eigenvalues and eigenvectors of a symmetric matrix $A \in \R^{n\times n}$ ranked in the descending order of eigenvalues. Set $X = (\v_1, \ldots, \v_r)$ and $Y = (\v_{r+1}, \ldots, \v_{n})$. Given a symmetric perturbation matrix $E$, let
\[
    \widehat E= (X, Y)^{\top} E (X, Y) = \left(
        \begin{array}{cc}
            \widehat E_{11} & \widehat E_{12} \\
            \widehat E_{21} & \widehat E_{22}
        \end{array}
    \right).
\]
Let $\|\cdot\|$ represent a consistent family of norms and set
\[
\gamma = \|\widehat E_{21}\|, \delta = \lambda_r - \lambda_{r+1} - \|\widehat E_{11}\| - \|\widehat E_{22}\|
\]
If $\delta > 0$ and $2\gamma < \delta$, then there exists a unique matrix $P \in \R^{(n-r)\times r}$ satisfying
$\|P\| < \frac{2\gamma}{\delta}$ such that
\begin{align*}
    X' &= (X + YP)(I + P^{\top}P)^{-1/2},\\
    Y' &= (Y - XP^{\top})(I + PP^{\top})^{-1/2},
\end{align*}
are the eigenvectors of $A + E$.
\end{thm}

Define
\begin{align*}
    \Theta &= \left(\vh_1, \ldots, \vh_r\right), \\
     \Phi &= \left(\varphi_1, \ldots, \varphi_r\right), \Phl = \left(\varphi_{r+1}, \ldots, \varphi_{N}\right). \label{eqn:Theta}
\end{align*}
The following theorem allows us to relate $\Theta$ with $\Phi$ and $\Phl$.
\begin{thm}\label{thm:3}
Assume
\[
\Delta = \frac{\lambda_r - \lambda_{r+1}}{N} > 3 \|L_N - L_m\|_{HS}.
\]
Then, there exists a matrix $P \in \R^{(N-r)\times r}$ satisfying
\[
    \|P\|_F \leq \frac{2\|L_N - L_m\|_{HS}}{\Delta - \|L_N - L_m\|_{HS}} \leq \frac{3\|L_N - L_m\|_{HS}}{\Delta},
\]
such that
\[
    \Theta = (\Phi + \Phl P)(I + P^{\top}P)^{-1/2}.
\]
\end{thm}
The proof can be found in Appendix B. As indicated by Theorem~\ref{thm:3}, when the eigengap $\Delta$ is sufficiently large, we have a small $\|P\|_F$ and therefore $\Theta \approx \Phi$, implying that the eigenfunctions $\{\vh_i\}_{i=1}^r$, computed based on the samples in $\Dh$, are good approximation of $\{\varphi_i\}_{i=1}^r$, the eigenfunctions of $L_N$. As a result, when the eigengap $\Delta$ is sufficiently large, we expect a small difference between $H_r$ and $\Hh_r$ because they are constructed based on eigenfunctions $\{\varphi_i\}_{i=1}^r$ and $\{\vh_i\}_{i=1}^r$, respectively. This is shown in the next theorem.
\begin{thm} \label{thm:4}
Assume
\[
\Delta = \frac{\lambda_r - \lambda_{r+1}}{N} > 3 \|L_N - L_m\|_{HS}.
\]
We have
\begin{eqnarray*}
    \|\Delta H\|_2 \leq \frac{4\|L_N - L_m\|_{HS}}{\Delta - \|L_N - L_m\|_{HS}} \leq \frac{6\|L_N - L_m\|_{HS}}{\Delta}.
\end{eqnarray*}
\end{thm}
The proof can be found in Appendix C. By putting the results from Theorem \ref{thm:concentration}, \ref{thm:1} and \ref{thm:4}, we have the final theorem for the approximation of the Nystr\"{o}m method measured in Frobenious norm.
\begin{thm} \label{thm:5}
Assume
\[
\Delta = \frac{\lambda_r - \lambda_{r+1}}{N} > 3 \|L_N - L_m\|_{HS}.
\]
We have
\[
\|K_r - \Kh_r\|_F\leq  \frac{4N\|L_N-L_m\|_{HS}}{\Delta - \|L_N-L_m\|_{HS}} \leq \frac{6N\|L_N-L_m\|_{HS}}{\Delta}.
\]
If the eigengap satisfies
\[
    \Delta= \Omega(1) > \frac{12\ln(2/\delta)}{\sqrt{m}},
\]
then, with a probability $1 - \delta$, we have
\begin{align*}
&\|K_r - \Kh_r\|_F 
\leq O\left(\frac{N}{\sqrt{m}}\right).
\end{align*}
\end{thm}
\begin{proof}
The proof is simply the combination of the results from Theorem \ref{thm:concentration}, \ref{thm:1} and \ref{thm:4}.
\begin{align*}
&\|K_r - \Kh_r\|_F \leq N\|\Delta H\|_2\leq \frac{4N\|L_N-L_m\|_{HS}}{\Delta - \|L_N-L_m\|_{HS}}\\
&\leq \frac{6N\|L_N-L_m\|_{HS}}{\Delta}\leq O\left(\frac{N}{\sqrt{m}}\right),
\end{align*}
where the third inequality follows $\|L_N-L_m\|_{HS}\leq \Delta/3$ and the last inequality follows from Theorem~\ref{thm:1}. Note that both conditions $\lambda_r > 0$ and $\lh_r > 0$, specified in Theorem~\ref{thm:1}, hold with a high probability. It is obvious that $\lambda_r > 0$ because $\lambda_r > \lambda_{r+1}$ and $\lambda_{r+1} \geq 0$. To show $\lh_r > 0$ holds with a high probability, we use the Lidskii's inequality~\citep{koltchinskii-2000-random}, i.e.,
\[
    \lh_r \geq \lambda_r - N\|L_N - L_m\|_{HS}.
\]
Since with a probability $1 - \delta$, $\lambda_r - \lambda_{r+1} \geq 3N\|L_N - L_m\|_{HS}$ holds, we have, with a probability $1 - \delta$
\[
\lh_r \geq \lambda_r - \frac{\lambda_r - \lambda_{r+1}}{3} = \frac{2}{3}\lambda_r + \frac{1}{3}\lambda_{r+1} > 0.
\]
\end{proof}

\paragraph{Remark} Besides the improved bound for the Nystr\"{o}m method, Theorem~\ref{thm:5} also explains the results shown in Figure~\ref{fig:app}. Since the additional approximation error $\|K - K_r\|_F - \|K - \Kh_r\|_F$ is upper bounded by $\|K_r - \Kh_r\|_F$, according to Theorem~\ref{thm:5}, we would expect the additional approximation error bound to be inversely related to the eigengap $\lambda_{r}-\lambda_{r+1}$, i.e. the larger the eigengap, the smaller the additional approximation error.

\section{Conclusion}
\label{sec-conclusion}
In this paper we tried to bridge the gap between effectiveness of Nystr\"{o}m method in practice and its poor theoretical approximation error bounds. In particular, in the case of large eigengap, we developed an improved bound for the approximation error of the Nystr\"{o}m method, based on the concentration inequality and the theory of matrix perturbation.  In the future, we plan to develop better bounds for the Nystr\"{o}m method that take into account the eigenvalues of kernel matrix which follow a power law.


\appendix
\section*{Appendix A: Proof of Theorem~\ref{thm:1}}
Since $K_r = KWK$ and $\Kh_r = K_b\Wh K_b^{\top}$, we have
\begin{align*}
&\|K W K - K_b \Wh K_b^{\top}\|_F^2  \\
&=\sum_{i,j=1}^N ([K W K]_{i,j} - [K_b \Wh K_b^{\top}]_{i,j})^2 \\
&= \sum_{i, j = 1}^N \langle \kappa(\x_i, \cdot), (H_r - \Hh_r) \kappa(\x_j, \cdot) \rangle^2 \\
& =  \sum_{i,j=1}^N \langle \Delta H\kappa(\x_i, \cdot), \kappa(\x_j, \cdot) \rangle  \langle \Delta H\kappa(\x_i, \cdot), \kappa(\x_j, \cdot) \rangle.
\end{align*}
Using the fact $\sum_{j=1}^N \langle f, \kappa(\x_j, \cdot) \rangle \langle f, \kappa(\x_j, \cdot)\rangle = N\langle f, L_N f \rangle$, we have
\begin{align*}
&\|K W K - K_b \Wh K_b^{\top}\|_F^2  \\
& =  N\sum_{i=1}^N \left\langle \Delta H \kappa(\x_i, \cdot),L_N\Delta H \kappa(\x_i, \cdot)  \right\rangle\\
& =  N\sum_{i=1}^N \langle \kappa(\x_i, \cdot), \Delta H L_N \Delta H \kappa(\x_i, \cdot) \rangle.
\end{align*}
We further simplify the expression by using the fact that for any linear operator $Z$, we have
\begin{align*}
\sum_{i=1}^N \langle \kappa(\x_i, \cdot), Z \kappa(\x_i, \cdot) \rangle = N\sum_{i=1}^N \langle \varphi_i, (Z L_N)\varphi_i \rangle.
\end{align*}
Using the above result with $Z = \Delta H L_N \Delta H$, we have
\begin{align*}
&\|K W K - K_b \Wh K_b^{\top}\|_F^2  \\
&= N^2\sum_{i=1}^N \langle \varphi_i, (\Delta H L_N \Delta H L_N)\varphi_i \rangle\\
& =  N\sum_{i=1}^N \lambda_i \langle \varphi_i, (\Delta H L_N \Delta H)\varphi_i \rangle\\
&\leq  N\lambda_1\sum_{i=1}^N  \langle \Delta H\varphi_i, L_N\Delta H \varphi_i\rangle\\
&=  \lambda_1 \sum_{i,j=1}^N \lambda_j \langle \varphi_j, \Delta H \varphi_i\rangle^2_{\Hk}=  \lambda_1 \sum_{i,j=1}^N \lambda_j \langle \varphi_i, \Delta H \varphi_j\rangle^2,
\end{align*}
where the last one equality follows equation~(\ref{eqn:eigL}). Define a matrix $ A = [\langle \varphi_i, \Delta H \varphi_j \rangle ]_{N\times N}$ and $D = \mbox{diag}(\lambda_1, \ldots, \lambda_N)$. We have
\begin{eqnarray*}
\lefteqn{\|KW K - K_b \Wh K_b^{\top}\|_F^2 \leq \lambda_1 \mbox{tr}(A D A)} \\
&\leq &\lambda_1 \|A\|_2^2 \sum_{i=1}^N \lambda_i \leq  \lambda_1 N \|\Delta H\|_2^2,
\end{eqnarray*}
where the last step follows from $\|A\|_2 = \|\Delta H\|_2$. 

\section*{Appendix B: Proof of Theorem~\ref{thm:3}}

Define matrix $B$ as
\[
B_{i,j} =\frac{1}{m} \sum_{k=1}^m \lh_k \langle \vh_k, \varphi_i \rangle \langle \vh_k, \varphi_j\rangle.
\]
Let $\z_i$ be the eigenvector of $B$ corresponding to eigenvalue $\lh_i/m$. It is straightforward to show that
\[
\z_i = (\langle \varphi_1, \vh_i\rangle_{\Hk}, \ldots, \langle \varphi_N, \vh_i \rangle_{\Hk})^{\top}, i\in[m],
\]
and therefore we have
\[
    \vh_i = \sum_{k=1}^N z_{i,k} \varphi_k, i\in[m], \mbox{ or } \Theta = (\Phi, \Phl)Z,
\]
where $Z=(\z_1,\cdots, \z_r)$.  To decide the relationship between $\{\vh_i\}_{i=1}^r$ and $\{\varphi_i\}_{i=1}^N$, we need to determine matrix $Z $. We define matrix $D = \mbox{diag}(\lambda_1/N, \ldots, \lambda_N/N)$ and matrix $E= B - D$, i.e.
\[
E_{i,j} = B_{i,j} - \lambda_i\delta_{i,j}/N = \langle \varphi_i, (L_m - L_N)\varphi_j \rangle_{\Hk}.
\]
Following the notation of Theorem~\ref{thm:2}, we define $X = (e_1, \ldots, e_r)$ and $Y = (e_{r+1}, \ldots, e_{N})$, where $e_1, \ldots, e_N$ are the canonical bases of $\R^N$, which are also eigenvectors of $D$. Define $\delta$ and $\gamma$ as follows
\begin{align*}
    \gamma & =  \sqrt{\sum_{i=1}^r \sum_{j=r+1}^N \langle \varphi_i, (L_N - L_m)\varphi_j \rangle^2_{\Hk}} \\
    \delta & = \Delta - \sqrt{\sum_{i,j=1}^r \langle \varphi_i, (L_N - L_m)\varphi_j\rangle^2_{\Hk}}\\
    & \hspace*{0.3in}- \sqrt{\sum_{i,j=r+1}^N \langle \varphi_i, (L_N - L_m)\varphi_j\rangle^2_{\Hk}}.
\end{align*}
It is easy to verify that $\gamma, \delta$ are defined with respect to the Frobenius norm of $\widehat E$  in Theorem~\ref{thm:2}. In order to apply the result in Theorem~\ref{thm:2}, we need to show $\delta > 0$ and $\gamma < \delta/2$. To this end, we need to provide the lower and upper bounds for $\gamma$ and $\delta$, respectively.
We first bound $\delta$ as
\begin{eqnarray*}
\delta - \Delta & \geq & - \sqrt{\sum_{i,j=1}^N \langle \varphi_i, (L_N - L_m)\varphi_j\rangle_{\Hk}^2} \\
& = & - \|L_N - L_m\|_{HS}.
\end{eqnarray*}
We then bound $\gamma$ as
\begin{eqnarray*}
\gamma & = & \sqrt{\sum_{i=1}^r\sum_{j=r+1}^N \langle \varphi_i, (L_N - L_m)\varphi_j\rangle_{\Hk}^2} \\
&\leq& \sqrt{\sum_{i=1}^N\sum_{j=1}^N \langle \varphi_i, (L_N - L_m)\varphi_j\rangle_{\Hk}^2}\\
&=& \|L_N - L_m\|_{HS}.
\end{eqnarray*}
Hence, when $\Delta > 3\|L_N - L_m\|_{HS}$, we have $\delta > 2\gamma > 0$, which satisfies the condition specified in Theorem~\ref{thm:2}. Thus, according to Theorem~\ref{thm:2}, there exists a $P \in \R^{(N - r)\times r}$ satisfying $\|P\| < 2\gamma/\delta$, such that
\[
Z = (\z_1, \ldots, \z_r) = (X + YP)(I + P^{\top}P)^{-1/2},
\]
implying
\[
    \Theta = (\Phi, \Phl)Z = (\Phi + \Phl P)(I + P^{\top}P)^{-1/2}.
\]

\section*{Appendix C: Proof of Theorem~\ref{thm:4}}
To bound $\|\Delta H\|_2$, it is sufficient to bound $\max_{\|f\|_{\Hk} \leq 1} \langle f, \Delta H f \rangle$. Consider any function $f(\cdot) = \sum_{i=1}^N f_i \varphi_i(\cdot)$, with $\|f\|_{\Hk} \leq 1$. Let $\f = (f_1, \ldots, f_N)^{\top}$. Evidently, we have $\|\f\|_2 \leq 1$. We have
\begin{align*}
\langle f, \Hh_r f \rangle = & \sum_{i=1}^r \sum_{a, b=1}^N f_a f_b \langle \varphi_a, \vh_i \rangle \langle \varphi_b, \vh_i \rangle = \|A^{\top}\f\|_2^2\\
\langle f, H_r f \rangle = & \sum_{i=1}^r \sum_{a, b=1}^N f_a f_b \langle \varphi_a, \varphi_i \rangle \langle \varphi_b, \varphi_i \rangle \\\\
&= \f^{\top}\left(\begin{array}{cc} I_{r\times r}  & 0 \\ 0 & 0\end{array} \right) \f.
\end{align*}
where $A = [\langle \varphi_i, \vh_j \rangle_{\Hk}]_{N\times m} = (\Phi, \Phl)^{\top}\Theta$. Since $\Delta > 3\|L_N - L_m\|_{HS}$, according to Theorem 4, there exists an matrix $P \in \R^{(N - r)\times r}$ satisfying
\[
\|P\|_F \leq \frac{2\|L_N - L_m\|_{HS}}{\Delta - \|L_N - L_m\|_{HS} },
\]such that
\[
\Theta = \left(\Phi + \Phl P\right)(I + P^{\top}P)^{-1/2}.
\]
Using the expression of $\Theta$, we compute $A$ as
\begin{align*}
A& = (\Phi, \Phl)^{\top}\Theta = (\Phi, \Phl)^{\top}\left(\Phi + \Phl P\right)(I + P^{\top}P)^{-1/2} \\
&= \left(\begin{array}{c} I \\ P \end{array} \right)(I + P^{\top}P)^{-1/2}.
\end{align*}
Thus, we have
\[
\langle f, \Delta H f \rangle = \f^{\top}\left(\left(\begin{array}{cc} I  & 0 \\ 0 & 0\end{array} \right) - AA^{\top}\right)\f = \f^{\top}C\f,
\]
where $C$ is given by
\begin{align*}
C&= \left(
\begin{array}{cc}
I - (I + P^{\top}P)^{-1} & (I + P^{\top}P)^{-1} P^{\top} \\
P(I + P^{\top}P)^{-1} & -P(I + P^{\top}P)^{-1}P^{-1}
\end{array}
\right) \\
&= \left(
\begin{array}{cc}
(I + P^{\top}P)^{-1}P^{\top}P & (I + P^{\top}P)^{-1} P^{\top} \\
P(I + P^{\top}P)^{-1} & -P(I + P^{\top}P)^{-1}P^{\top}
\end{array}
\right).
\end{align*}
Rewrite $\f = (\f_a, \f_b)$ where $\f_a \in \R^r$ includes the first $r$ entries in $\f$ and $\f_b$ includes the rest of the entries in $\f$. We have
\begin{align*}
&\f^{\top}C\f \leq \left(\|\f_a\|_2^2 + \|\f_b\|_2^2\right)\|P(I + P^{\top}P)^{-1}P^{\top}\|_2 \\
&+ 2\|\f_a\|\|\f_b\|\left\|(I + P^{\top}P)^{-1}P\right\|_2 \\
& \leq (\|\f_a\|_2 + \|\f_b\|_2)^2\cdot\\
&\hspace*{0.1in}\max\left(\left\|P(I + P^{\top}P)^{-1} P^{\top}\right\|_2, \left\|(I + P^{\top}P)^{-1} P^{\top}\right\|_2 \right) \\
& \leq  2\max\left(\left\|P(I + P^{\top}P)^{-1} P^{\top}\right\|_2, \left\|(I + P^{\top}P)^{-1} P^{\top}\right\|_2 \right).
\end{align*}
Since $\|P\|_2\leq \|P\|_F\leq 1$ because $\Delta > 3 \|L_N - L_m\|_{HS}$ and 
\begin{align*}
\left\|P(I + P^{\top}P)^{-1} P^{\top}\right\|_2 & \leq  \|P\|^2_2, \\
\left\|(I + P^{\top}P)^{-1} P^{\top}\right\|_2 &  \leq \|P\|_2,
\end{align*}
we have
\begin{align*}
   & \max\limits_{\|f\|_{\Hk}\leq 1} \langle f, \Delta H f \rangle = \f^{\top}C\f\\
    & \leq 2\|P\|_2\leq 2\|P\|_F\leq\frac{4\|L_N - L_m\|_{HS}}{\Delta - \|L_N - L_m\|_{HS}}.
\end{align*}

\bibliographystyle{icml2012}
\bibliography{nystrom}

\begin{thebibliography}{23}
\providecommand{\natexlab}[1]{#1}
\providecommand{\url}[1]{\texttt{#1}}
\expandafter\ifx\csname urlstyle\endcsname\relax
  \providecommand{\doi}[1]{doi: #1}\else
  \providecommand{\doi}{doi: \begingroup \urlstyle{rm}\Url}\fi

\bibitem[Azran \& Ghahramani(2006)Azran and Ghahramani]{Azran:Ghahramani:06}
Azran, Arik and Ghahramani, Zoubin.
\newblock Spectral methods for automatic multiscale data clustering.
\newblock In \emph{Proceedings of the IEEE Conference on Computer Vision and
  Pattern Recognition (CVPR 2006)}, 2006.

\bibitem[Bach \& Jordan(2003)Bach and Jordan]{Bach:CSD-03-1249}
Bach, Francis~R. and Jordan, Michael~I.
\newblock Learning spectral clustering.
\newblock Technical Report UCB/CSD-03-1249, EECS Department, University of
  California, Berkeley, Jun 2003.

\bibitem[Belabbas \& Wolfe(2009)Belabbas and Wolfe]{belabbas-2009-spectral}
Belabbas, M.-A. and Wolfe, P.~J.
\newblock Spectral methods in machine learning and new strategies for very
  large data sets.
\newblock \emph{Proceedings of the National Academy of Sciences of the USA},
  106:\penalty0 369--374, 2009.

\bibitem[Belkin \& Niyogi(2001)Belkin and Niyogi]{Belkin:2001}
Belkin, Mikhail and Niyogi, Partha.
\newblock Laplacian eigenmaps and spectral techniques for embedding and
  clustering.
\newblock \emph{Advances in Neural Information Processing Systems}, pp.\
  585--591, 2001.

\bibitem[Chitta et~al.(2011)Chitta, Jin, Havens, and Jain]{ChittaJHJ11}
Chitta, Radha, Jin, Rong, Havens, Timothy~C., and Jain, Anil~K.
\newblock Approximate kernel k-means: solution to large scale kernel
  clustering.
\newblock In \emph{Proceedings of the 17th ACM SIGKDD International Conference
  on Knowledge Discovery and Data Mining (KDD)}, pp.\  895--903, 2011.

\bibitem[Cortes et~al.(2010)Cortes, Mohri, and Talwalkar]{cortes-2010-nystrom}
Cortes, Corinna, Mohri, Mehryar, and Talwalkar, Ameet.
\newblock On the impact of kernel approximation on learning accuracy.
\newblock \emph{Journal of Machine Learning Research - Proceedings Track},
  9:\penalty0 113--120, 2010.

\bibitem[Drineas \& Mahoney(2005)Drineas and Mahoney]{Drineas05onthe}
Drineas, Petros and Mahoney, Michael~W.
\newblock On the nystrom method for approximating a gram matrix for improved
  kernel-based learning.
\newblock \emph{Journal of Machine Learning Research}, 6:\penalty0 2005, 2005.

\bibitem[Fowlkes et~al.(2004{\natexlab{a}})Fowlkes, Belongie, Chung, and
  Malik]{Fowlkes04spectralgrouping}
Fowlkes, Charless, Belongie, Serge, Chung, Fan, and Malik, Jitendra.
\newblock Spectral grouping using the nystrom method.
\newblock \emph{IEEE Transactions on Pattern Analysis and Machine
  Intelligence}, 26:\penalty0 2004, 2004{\natexlab{a}}.

\bibitem[Fowlkes et~al.(2004{\natexlab{b}})Fowlkes, Belongie, Chung, and
  Malik]{Fowlkes:2004:SGU:960255.960312}
Fowlkes, Charless, Belongie, Serge, Chung, Fan, and Malik, Jitendra.
\newblock Spectral grouping using the nystr\"{o} method.
\newblock \emph{IEEE Trans. Pattern Anal. Mach. Intell.}, pp.\  214--225,
  2004{\natexlab{b}}.

\bibitem[Hough et~al.(2006)Hough, Krishnapur, Peres, and
  Virag]{hough-2006-determinantal}
Hough, J.~Ben, Krishnapur, Manjunath, Peres, Yuval, and Virag, Balint.
\newblock Determinantal processes and independence.
\newblock \emph{Probability Surveys}, 3:\penalty0 206--229, 2006.

\bibitem[Koltchinskii \& Gine(2000)Koltchinskii and
  Gine]{koltchinskii-2000-random}
Koltchinskii, Vladimir and Gine, Evarist.
\newblock Random matrix approximation of spectra of integral operators.
\newblock \emph{Bernoulli}, 6:\penalty0 113 -- 167, 2000.

\bibitem[Kumar et~al.(2009)Kumar, Mohri, and Talwalkar]{kuma-2009-sampling}
Kumar, S., Mohri, M., and Talwalkar, A.
\newblock Sampling techniques for the nystrom method.
\newblock In \emph{Proceedings of Conference on Artificial Intelligence and
  Statistics}, pp.\  304 -- 311, 2009.

\bibitem[Luxburg(2007)]{Luxburg:2007:tutorial}
Luxburg, Ulrike.
\newblock A tutorial on spectral clustering.
\newblock \emph{Statistics and Computing}, 17:\penalty0 395--416, December
  2007.

\bibitem[Platt(2004)]{Platt04fastembedding}
Platt, John~C.
\newblock Fast embedding of sparse music similarity graphs.
\newblock In \emph{Advances in Neural Information Processing Systems 16}, pp.\
  2004. MIT Press, 2004.

\bibitem[Sch\"{o}lkopf et~al.(1998)Sch\"{o}lkopf, Smola, and
  M\"{u}ller]{Scholkopf:1998:NCA:295919.295960}
Sch\"{o}lkopf, Bernhard, Smola, Alexander, and M\"{u}ller, Klaus-Robert.
\newblock Nonlinear component analysis as a kernel eigenvalue problem.
\newblock \emph{Neural Comput.}, pp.\  1299--1319, 1998.

\bibitem[Shi et~al.(2009)Shi, Belkin, and Yu]{DS_AOS_09}
Shi, Tao, Belkin, Mikhail, and Yu, Bin.
\newblock Data spectroscopy: eigenspace of convolution operators and
  clustering.
\newblock \emph{The Annals of Statistics}, 37, 6B:\penalty0 3960--3984, 2009.

\bibitem[Silva \& Tenenbaum(2003)Silva and Tenenbaum]{silva-2003-gloal}
Silva, Vin~De and Tenenbaum, Joshua~B.
\newblock Global versus local methods in nonlinear dimensionality reduction.
\newblock In \emph{Advances in Neural Information Processing Systems 15}, pp.\
  705--712, 2003.

\bibitem[Smale \& Zhou(2009)Smale and Zhou]{smale-2009-geometry}
Smale, Steve and Zhou, Ding-Xuan.
\newblock Geometry on probability spaces.
\newblock \emph{Constr Approx}, 30:\penalty0 311--323, 2009.

\bibitem[Stewart \& guang Sun(1990)Stewart and guang Sun]{Stewart90}
Stewart, G.~W. and guang Sun, Ji.
\newblock \emph{Matrix Perturbation Theory}.
\newblock Academic Press, 1990.

\bibitem[Talwalkar \& Rostamizadeh(2010)Talwalkar and
  Rostamizadeh]{talwalkar-2010-matrix}
Talwalkar, Ameet and Rostamizadeh, Afshin.
\newblock Matrix coherence and the nystrom method.
\newblock In \emph{Proceedings of Conference on Uncertainty in Artificial
  Intelligence (UAI)}, 2010.

\bibitem[Talwalkar et~al.(2008)Talwalkar, Kumar, and
  Rowley]{talwalkar-2008-large}
Talwalkar, Ameet, Kumar, Sanjiv, and Rowley, Henry~A.
\newblock Large-scale manifold learning.
\newblock In \emph{IEEE Computer Society Conference on Computer Vision and
  Pattern Recognition (CVPR 2008)}, 2008.

\bibitem[Williams \& Seeger(2001)Williams and Seeger]{Williams01usingthe}
Williams, Christopher and Seeger, Matthias.
\newblock Using the nystrom method to speed up kernel machines.
\newblock In \emph{Advances in Neural Information Processing Systems 13}, pp.\
  682--688. MIT Press, 2001.

\bibitem[Zhang et~al.(2008)Zhang, Tsang, and Kwok]{kai-2008-improved}
Zhang, Kai, Tsang, Ivor~W., and Kwok, James~T.
\newblock Improved nystrom low-rank approximation and error analysis.
\newblock In \emph{Proceedings of International Conference on Machine Learning
  (ICML 2008)}, 2008.

\end{thebibliography}

\end{document}